\documentclass[sigconf,anonymous=false,nonacm]{acmart}  %
\usepackage{natbib}
\AtBeginDocument{%
  \providecommand\BibTeX{{%
    \normalfont B\kern-0.5em{\scshape i\kern-0.25em b}\kern-0.8em\TeX}}}

\setcopyright{acmcopyright}
\copyrightyear{2023}
\acmYear{2023}
\acmDOI{XXXXXXX.XXXXXXX}

\acmConference[SIGSPATIAL '23]{The 31st ACM SIGSPATIAL International Conference on Advances in Geographic Information Systems}{Nov 13--16, 2023}{Hamburg, Germany}
\acmPrice{15.00}
\acmISBN{978-1-4503-XXXX-X/18/06}

\begin{document}

\title{TAP: A Comprehensive Data Repository for Traffic Accident Prediction in Road Networks}


\author{Baixiang Huang}
\email{bhuang15@hawk.iit.edu}
\affiliation{
  \institution{Illinois Institute of Technology}
  \streetaddress{10 W 35th Street}
  \city{Chicago}
  \state{Illinois}
  \country{USA}
  \postcode{60616}
}

\author{Bryan Hooi}
\email{bhooi@comp.nus.edu.sg}
\affiliation{
  \institution{National University of Singapore}
  \streetaddress{21 Lower Kent Ridge Road}
  \country{Singapore}
  \postcode{119077}
}

\author{Kai Shu}
\email{kshu@iit.edu}
\affiliation{
  \institution{Illinois Institute of Technology}
  \streetaddress{10 W 35th Street}
  \city{Chicago}
  \state{Illinois}
  \country{USA}
  \postcode{60616}
}

\renewcommand{\shortauthors}{Baixiang Huang, Bryan Hooi, and Kai Shu}

\begin{abstract}
Road safety is a major global public health concern. Effective traffic crash prediction can play a critical role in reducing road traffic accidents. However, Existing machine learning approaches tend to focus on predicting traffic accidents in isolation, without considering the potential relationships between different accident locations within road networks. To incorporate graph structure information, graph-based approaches such as Graph Neural Networks (GNNs) can be naturally applied. However, applying GNNs to the accident prediction problem faces challenges due to the lack of suitable graph-structured traffic accident datasets. To bridge this gap, we have constructed a real-world graph-based Traffic Accident Prediction (TAP) data repository, along with two representative tasks: accident occurrence prediction and accident severity prediction. With nationwide coverage, real-world network topology, and rich geospatial features, this data repository can be used for a variety of traffic-related tasks. We further comprehensively evaluate eleven state-of-the-art GNN variants and two non-graph-based machine learning methods using the created datasets. Significantly facilitated by the proposed data, we develop a novel Traffic Accident Vulnerability Estimation via Linkage (\textsc{TRAVEL}) model, which is designed to capture angular and directional information from road networks. We demonstrate that the proposed model consistently outperforms the baselines. The data and code are available on GitHub \footnote{https://github.com/baixianghuang/travel}.
\end{abstract}



\maketitle

\section{Introduction}
Road traffic accidents \footnote{In this paper, we use the terms "accident" and "crash" interchangeably following convention. The Oxford English Dictionary defines an accident as an event that happens unexpectedly or without apparent cause. However, many roadway “accidents” are the direct result of some environmental factors or unsafe driving behaviors.} are the leading cause of death for people aged five to twenty-nine years globally \cite{world2018global}. The U.S. traffic fatality rate has experienced an alarming 19 percent surge from 2019 to 2021, marking the highest number of road deaths in the U.S. Fatality Analysis Reporting System's history since 2005 \cite{trip2022trafficfatality,nhtsa2022fatareport}. Therefore, understanding and mitigating traffic crashes is an imminent task.

\begin{figure}[t]
    \centering
    \includegraphics[width=0.45\textwidth]{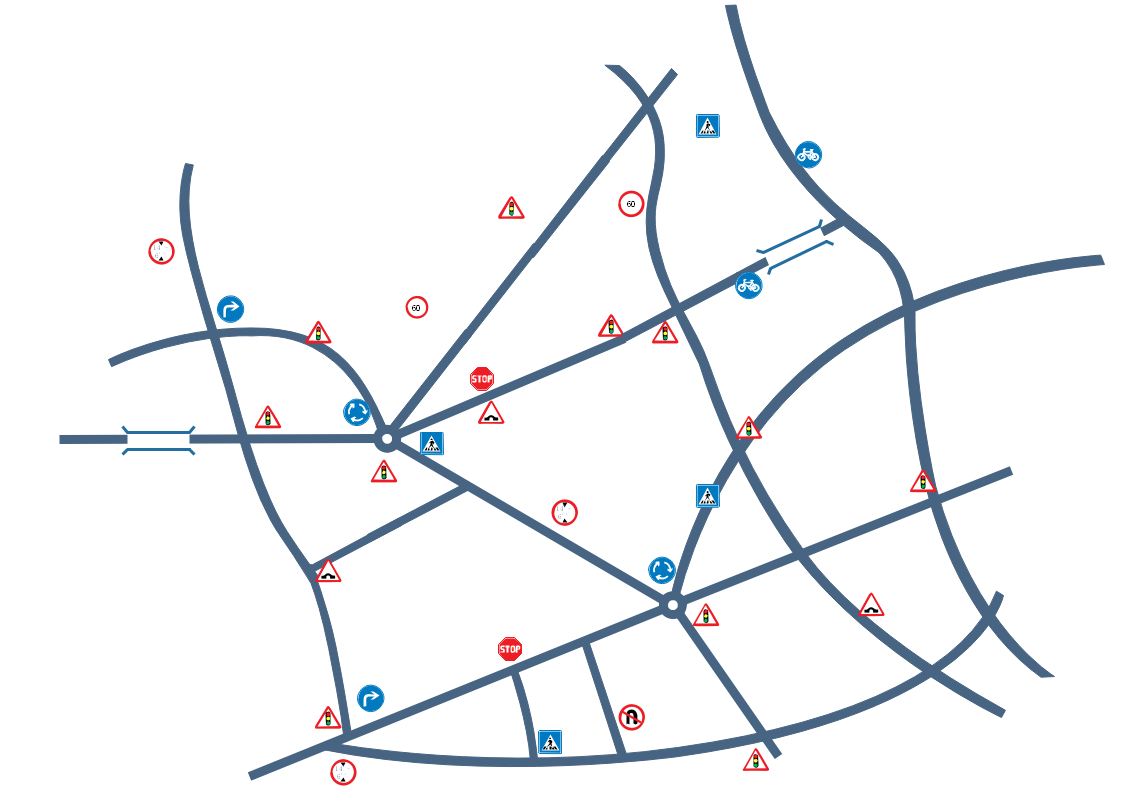}
    \caption{An illustration of a road network with various environmental features (e.g., traffic signs, speed limits, and road types).}
    \label{fig:intersection-demo1}
    \Description{A road network with different environmental features.}
\end{figure}

Roadway accidents are preventable and the direct result of environmental factors or unsafe driving behaviors. In this work, we focus on the environmental risk factors of road traffic accidents: in particular, can we predict how risk-prone a traffic intersection is based only on street map data, such as the road geometry, nearby highways, and landmarks? Traffic crash prediction at a fine-grained spatial scale (e.g., a traffic intersection) can help governments to mitigate traffic risks, such as informing the design of future road networks and planning accident response systems with the awareness of risk-prone accident hotspots. According to the U.S. Federal Highway Administration's Proven Safety Countermeasures initiative \cite{web23countermeasure}, road design plays a critical role among countermeasures and strategies in reducing roadway deaths and serious injuries.

As shown in Figure \ref{fig:intersection-demo1}, a city road network has a variety of environmental features that could potentially be correlated with accidents - such as road types and road lengths. The network has roundabouts, bridges, one-way roads, residential roads, tertiary roads, and motorways. Roads also have other attributes, such as stop signs and speed limits. Different road types and conditions may pose different risks for crashes. For instance, a high-volume freeway is likely to have more roadway crashes than a less traveled two-lane residential road because it has more traffic and a higher speed limit. Moreover, road characteristics such as turning radius and direction also affect road safety. In general, sharper road curves are more dangerous \cite{othman2009identifying}. Hence, our goal is to design an algorithm that learns how geospatial information should be used to predict the risk of accidents in each road intersection. 

The majority of existing machine learning approaches for predicting traffic accidents uses grid-based approaches. Intuitively, the road network contains rich information across different locations that have the potential to enhance traffic accident prediction, e.g., the shape or angle of roads leading to an intersection. In addition, the angular and direction information in the network structures can also be important. Therefore, graph-based approaches allow surrounding information to be taken into account by aggregating information from neighbors multiple hops away, compared to existing grid-based approaches, which only use information from within each grid cell. Moreover, graph-based approaches provide the flexibility to take advantage of geospatial data sources such as OpenStreetMap (OSM) \cite{web23osm}. Maps from OSM are structured as graphs, where nodes represent intersections and dead-end nodes, and edges represent roads. Graph-based methods enable better prediction of risky intersections with map data.


To learn from graph data, Graph Neural Networks (GNNs) are becoming increasingly popular \cite{hamilton2017representation}. Despite the potential benefits of using GNNs for traffic accident prediction, several challenges still need to be addressed. Many approaches fail to accurately reflect real-world road structures due to the lack of benchmark datasets with such information \cite{li2017speedrandomwalkgru,zhou2020spatiotemporalaccident}. The majority of existing accident datasets are not organized with a proper graph structure, making it infeasible to apply GNNs. Existing traffic accident datasets \cite{web23nyc,web22nys,web23chicago} have limited coverage and geospatial features. To facilitate the development of graph-based machine learning methods, we construct and release a Traffic Accident Prediction (TAP) data repository, with two prediction tasks (accident occurrence prediction and accident severity prediction). 

To build the data repository, We collect raw accident records, street geospatial data, and graph structure information. The collected accident coordinates are reversely geocoded. Next, the crash data are integrated with graph structure and geospatial features, which are then reorganized and preprocessed. In the proposed TAP repository, datasets are organized into city-level (TAP-city) and state-level (TAP-state), covering over 1,000 U.S. cities and 49 states. This repository allows users to study traffic-related problems without significant preprocessing effort. We will continuously update and expand this dataset as we collect more crash data and other auxiliary features.

We further develop a novel architecture called Traffic Accident Vulnerability Estimation via Linkage (\textsc{TRAVEL}). Different from existing GNNs, where each node aggregates from its neighbors using simple functions, ours aggregates in a way that captures both the angles and directions of roads adjacent to a node. \textsc{TRAVEL}’s graph convolution layers consist of two components: modeling the angles between roads and measuring the direction information simultaneously. We demonstrate superior performance by leveraging the TAP repository. The key contributions of this work are as follows:
\begin{itemize}
    \item We formulate the graph-based traffic accident prediction as a node prediction problem, which aims to predict accident occurrences or accident severities over a road graph. 
    
    \item We construct and release a Traffic Accident Prediction (TAP) data repository that significantly simplifies the application of graph-based machine learning methods for traffic crash prediction and analysis. The proposed datasets are graph-based and contain real-world geospatial features.
    
    \item We propose a new GNN architecture, \textsc{TRAVEL}, which can capture angular and directional information from road networks. We comprehensively evaluate our model against MLP, XGBoost, and eleven state-of-the-art GNN baselines. We validate that our proposed model consistently has the best performance on the benchmark datasets.
\end{itemize}

\section{Related Work}
\subsection{Traffic Accident Prediction}
Roadway crash prediction is often formulated as a classification or regression problem. Early work proposed Poisson, Negative Binomial, and Negative Multinomial regression models to predict the number of accidents over a discretized grid \cite{oh2006regressionaccident,caliendo2007regressionaccident}. \citet{najjar2017imgaccident} train traditional Convolutional Neural Networks (CNNs) on satellite images of traffic accidents to produce a traffic risk map. These studies convert road networks to regular 2-D grids because traditional convolutional operations handle spatial correlations over such grids. In contrast, our approach not only takes full advantage of graph structures of road networks but also incorporates real-world environmental features. 

The traffic accident prediction problem has also been formulated as a spatiotemporal forecasting task. Early studies apply k-nearest neighbors \cite{lv2009knn}, k-means clustering, and logistic regression \cite{park2016kmeans}. \citet{chen2016realtimeaccident} use human mobility features obtained from stacked denoising autoencoders to infer traffic risk. \citet{yu2017lstm} combine Long Short-Term Memory (LSTM) and stacked autoencoders for post-accident condition prediction. \citet{ren2018lstm} propose a model based on Recurrent Neural Networks (RNNs) to capture spatial and temporal patterns of traffic accident frequency. \citet{zhou2020riskoracle} use deep learning for spatiotemporal accident prediction over a rectangular grid. With the emergence of graph-based deep learning approaches, some work applies GNNs to this task \cite{yu2021spatiotemporal,zhou2020spatiotemporalaccident}. However, The majority of existing GNN models for accident prediction use generated graphs instead of the real-world road networks. For example, \citet{li2017speedrandomwalkgru} construct a directed weighted graph using thresholded Gaussian kernel weighting function \cite{shuman2013edgeweight}. \citet{zhou2020spatiotemporalaccident} formulate graphs by dividing the study area into grids and adding edges between grids with strong correlations.

\subsection{Traffic Spatiotemporal Forecasting}
In addition to accident prediction, some research focuses on other spatiotemporal traffic prediction problems such as traffic speed or volume forecasting. Classic data-driven traffic forecasting methods include the Auto-Regressive Integrated Moving Average (ARIMA) model and Kalman filtering \cite{liu2011arima,lippi2013kalman}. Deep learning models for these tasks typically use RNNs and CNNs to capture temporal and spatial correlations, respectively \cite{yu2017lstm,zhang2017resnettraffic}. \citet{yao2019stdn} proposed a spatial-temporal dynamic network that uses a periodically shifted attention mechanism.

Graph-based approaches, such as bidirectional random walks \cite{li2017diffusion} and graph convolution operators \cite{yu2017gcntraffic,lv2018gcntrafficspeed,cui2019gcntraffic,chen2020bicomponentgcn}, are also used to capture spatial dependencies in traffic forecasting. \citet{zheng2020gman} further proposed a graph multi-attention network that adapts an encoder-decoder architecture and attention mechanisms.

Our goal is to develop a method that can be applied directly to readily available street map data from OpenStreetMap, which would make it much easier for users to apply our method to new cities. Geospatial data for new cities is often already available, while collecting temporal or other additional data would require specialized data collection efforts. Although temporal tasks are not discussed in this paper, we included accident timestamps in our datasets to support future work that considers temporal information.

\section{Preliminaries}
\subsection{Problem Statement}
We model a road network as a weighted directed graph $G = (V, E)$, where vertices in $V$ represent endpoints (intersections and dead-end nodes) in the road network, while edges in $E$ represent roads. The graph also has a set of node features $\mathbf{x}_v \in \mathbb{R}^{d_v}$, where $d_v$ is the number of node features, and edge attributes $\mathbf{e}_{uv} \in \mathbb{R}^{d_e}$ for $(u,v) \in E$, where $d_e$ is the number of edge attributes. 

In the `accident occurrence' prediction task, we aim to output a binary prediction indicating whether a vertex has traffic accidents or not. In the `accident severity' prediction task, we bucketize the average severity of accidents at each node, thereby grouping the nodes into a finite number of classes. Then, we aim to predict the class of each node. Therefore, both of these tasks are formulated as node classification problems. We focus on accidents around intersections because an analysis of the US-Accident dataset shows that the majority of the accidents took place near intersections \cite{moosavi2019usaccident}. Each node $v$ has its label denoted as $y_v$.

\subsection{Basic GNN Framework}
A GNN model uses node features $\mathbf{x}_v$ to learn the representation of a node. In each layer, each node aggregates representations of its neighbors and updates the representation of itself. The $k$-th layer of a basic GNN is:

\begin{equation}
    \mathbf{h}_v^{(k)} = \mathsf{Update}^{(k)} (\mathbf{h}_v^{(k-1)}, \mathbf{m}_{\mathcal{N}(v)}^{(k)})
\end{equation}

\begin{equation}
    \mathbf{m}_{\mathcal{N}(v)}^{(k)} = \mathsf{Aggr}^{(k)} (\{\mathbf{h}_u^{(k-1)}, \forall u \in \mathcal{N}(v)\})
\end{equation}

where $\mathbf{h}_v^{(k)}$ is the node embedding of node $v$ in the $k$-th layer. Initial $0$-th layer embeddings $\mathbf{h}_v^{(0)}$ are equal to node features $\mathbf{x}_v$. $\mathbf{m}_{\mathcal{N}(v)}$ denotes the aggregated message from node $v$'s neighborhood $\mathcal{N}(v) = \{u: (u,v) \in E\}$. $\mathsf{Update}$ is a neural network that updates the representation of $v$, and $\mathsf{Aggr}$ is a function that aggregates representations of $\mathcal{N}(v)$.

There are some GNN variants that support message passing with multi-dimensional edge features. For example, in MPNN \cite{gilmer2017mpnn}, each layer is defined as:
\begin{equation} \label{eq_mpnn}
    \mathbf{h}_v^{(k)} = \sigma (\mathbf{W}^{(k)} \mathbf{h}_v^{(k-1)} + \sum_{u \in \mathcal{N}(v)} \mathbf{h}_u^{(k-1)} \cdot \mathsf{Net}(\mathbf{e}_{vu}))
\end{equation}
where $\sigma$ is a nonlinear function, $\mathbf{W}^{(k)}$ represents a trainable weight matrix, and $\mathsf{Net}$ denotes a neural network that matches the dimensionality of the edge features $\mathbf{e}_{vu}$ to that of the neighbor embedding $\mathbf{h}_u^{(k-1)}$.

\section{Data Repository Construction} \label{section-dataset}

\begin{table}
\small
\centering
\caption{Statistics of TAP-city (1,000 city-level datasets).}
\label{table:data-stat-city}
\begin{tabular}{lcccc} 
\toprule
\textbf{Measure}      & \textbf{\# nodes} & \textbf{\# edges}  & \begin{tabular}[c]{@{}c@{}}\textbf{Avg node}\\\textbf{degree}\end{tabular} & \begin{tabular}[c]{@{}c@{}}\textbf{\% accident}\\\textbf{nodes}\end{tabular}  \\ 
\midrule
mean  & 3,051.84 & 7,782.13  & 2.56 & 8.84  \\
std   & 5,658.60 & 14,469.31 & 0.19 & 7.18  \\
min   & 50       & 95        & 1.46 & 0.45  \\
0.25  & 628.00   & 1,615.75  & 2.43 & 4.47  \\
0.5   & 1,308.50 & 3,372.50  & 2.54 & 6.75  \\
0.75  & 2,948.75 & 7,523.00  & 2.69 & 10.64 \\
max   & 59,711   & 148,937   & 3.21 & 54.84 \\
\bottomrule
\end{tabular}
\end{table}

\begin{table}
\small
\centering
\caption{Statistics of TAP-state (49 state-level datasets).}
\label{table:data-stat-state}
\begin{tabular}{lcccc} 
\toprule
\textbf{Measure}      & \textbf{\# nodes} & \textbf{\# edges}  & \begin{tabular}[c]{@{}c@{}}\textbf{Avg node}\\\textbf{degree}\end{tabular} & \begin{tabular}[c]{@{}c@{}}\textbf{\% accident}\\\textbf{nodes}\end{tabular}  \\ 
\midrule
mean  & 346,537.14 & 876,531.86 & 2.54 & 3.45   \\
std   & 291,425.12 & 735,376.82 & 0.14 & 3.66   \\
min   & 9,941      & 26,769     & 2.33 & 0.16   \\
0.25  & 142,558    & 350,503    & 2.44 & 1.41   \\
0.5   & 308,919    & 770,361    & 2.52 & 1.94   \\
0.75  & 448,635    & 1,054,027  & 2.58 & 4.57   \\
max   & 1,608,908  & 4,099,325  & 2.99 & 21.14  \\   
\bottomrule
\end{tabular}
\end{table}

\begin{table*}
\small
\centering
\caption{Comparison with existing traffic prediction datasets.}
\label{table:data-compare}
\begin{tabular}{llcclll} 
\toprule
\textbf{Dataset} & \textbf{Type} & \textbf{Geospatial} & \textbf{Graph-based} & \textbf{Coverage} & \textbf{Time} \\ 
\midrule
PeMS \cite{chen2001datasetpems} & Flow and speed & & & State-wide & 2001 \\
METR-LA \cite{jagadish2014datasetmetrla} & Flow and speed & & & County-wide & 2014 \\
Chicago Traffic Crashes \cite{web23chicago}  & Accident       &            &            & City-wide (0.70M rows) & 2023 \\
NY City Motor Vehicle Collisions \cite{web23nyc} & Accident       &            &            & City-wide (1.98M rows) & 2023 \\
NY State Motor Vehicle Crashes \cite{web22nys} & Accident       &            &          & State-wide (3.51M rows) & 2022 \\
UK Traffic Accidents \cite{web16ukaccident} & Accident       &            &             & Nation-wide (1.60M rows) & 2016 \\
TAP  & Accident  & \checkmark  & \checkmark  & Nation-wide (16.98M nodes and 42.95M edges) & 2023 \\
\bottomrule
\end{tabular}
\end{table*}

A major obstacle that explains why it is difficult to apply GNNs to traffic accident prediction is the lack of graph-based datasets. We describe how we construct the Traffic Accident Benchmark (TAP) repository in this section. The city-level and state-level datasets are denoted as TAP-city and TAP-state, respectively. Numerical measures of the data repository can be found in Table \ref{table:data-stat-city} and Table \ref{table:data-stat-state}. Table \ref{table:dataset-stat} shows the statistics of 6 sample datasets. We provide a convenient and user-friendly interface. Initializing our datasets will automatically download the preprocessed files, the result of which can be easily plugged into existing GNNs. 

We provide a list of existing datasets and compare them with our repository in Table \ref{table:data-compare}. In general, popular traffic accident datasets include Chicago Traffic Crashes \cite{web23chicago}, New York City Motor Vehicle Collisions \cite{web23nyc}, New York State Motor Vehicle Crashes \cite{web22nys}, and UK Traffic Accidents \cite{web16ukaccident}. The TAP data repository is graph-based, contains rich geospatial features, and has comprehensive geographical coverage. Additionally, there are other traffic flow or speed datasets such as PeMS, which is collected by California Transportation Agencies and provides real-time traffic speed records in California \cite{chen2001datasetpems}. Another dataset is METR-LA, a traffic dataset that contains traffic information collected from loop detectors in the highways of Los Angeles County \cite{jagadish2014datasetmetrla}. Traffic accident data are typically collected from police reports \cite{lloyd2016britainreport,web23nyc}, which are often made publicly available. However, traffic flow and speed data are less available as traffic monitoring devices are not widely available or completely prohibited by laws \cite{zhu2019availability,web2023lawbystate}.

\begin{table}
\small
\centering
\caption{Statistics of six sample datasets.}
\label{table:dataset-stat}
\begin{tabular}{lcccc} 
\toprule
\textbf{City}      & \textbf{\# nodes} & \textbf{\# edges}  & \begin{tabular}[c]{@{}c@{}}\textbf{Avg node}\\\textbf{degree}\end{tabular} & \begin{tabular}[c]{@{}c@{}}\textbf{\% accident}\\\textbf{nodes}\end{tabular}  \\ 
\midrule
Houston        & 59,711 & 148,937 & 2.49 & 22.10 \\
New York       & 55,404 & 140,005 & 2.53 & 8.25  \\
Los Angeles    & 49,251 & 135,547 & 2.75 & 13.01 \\
Dallas         & 36,150 & 92,348  & 2.55 & 25.79 \\
Miami          & 8,461  & 22,648  & 2.68 & 13.31 \\
Orlando        & 7,513  & 18,216  & 2.42 & 30.17 \\
\bottomrule
\end{tabular}
\end{table}

\begin{figure}[t]
    \centering
    \includegraphics[width=0.45\textwidth]{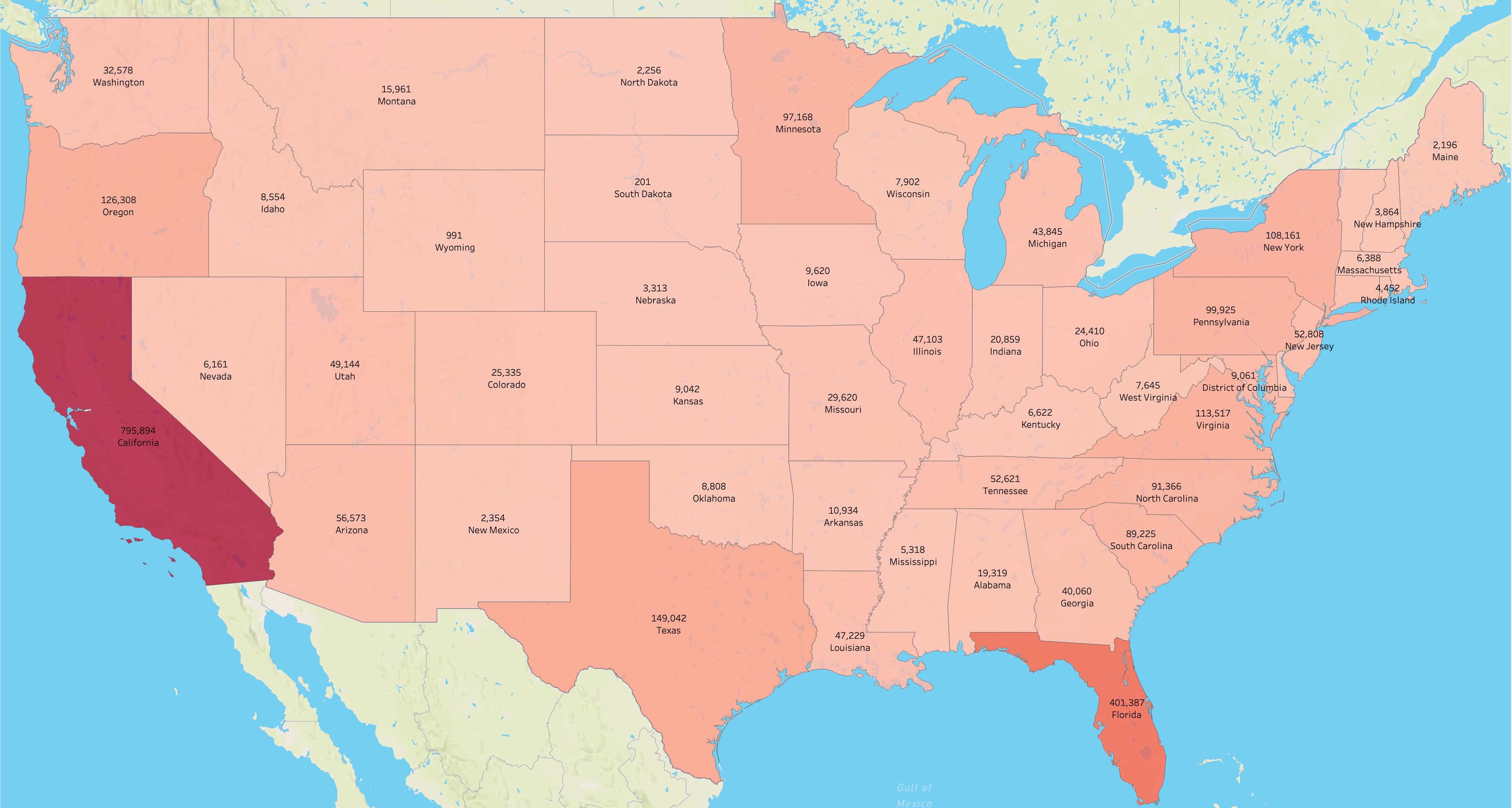}
    \caption{Total number of accidents by states between 2016 and 2021.}
    \label{fig:accidents-cnt-state}
\end{figure}


\subsection{Data Collection}
The raw accident events come from Microsoft Bing Map Traffic \cite{webbingmap}, which are collected in the US-Accidents dataset \cite{web22usaccident}. It contains about 2.8 million instances of traffic accident data between January 2016 and December 2021. Roughly 32\% of accidents occurred on or near local roads (e.g., streets, avenues, and boulevards), and about 40\% took place on or near high-speed roads (e.g., highways, interstates, and state roads) \cite{moosavi2019usaccident}. We use the data across the entire five-year time period. Figure \ref{fig:accidents-cnt-state} illustrates the total number of accidents by states over five years. The top 10 states in the U.S. with the highest total crash count are California, Florida, Texas, Oregon, Virginia, New York, Pennsylvania, Minnesota, North Carolina, and South Carolina. 

We use OpenStreetMap as our data source for geospatial data, a collaborative initiative that provides a freely available geospatial database. The OSM data contain rich environmental features such as road type, road length, bridge type, and the number of lanes. These features are also keys for OSM tags, which describe the specific attributes of map elements (nodes, ways, or relations). We use the OSMnx \cite{boeing2017osmnx} package to collect raw geospatial data from OSM. OSM includes walkable, drivable, and bikeable urban road data. Since most accidents occur on the drivable networks, we only use drivable public road data (private-access or service roads not included). In a road network, nodes are points such as intersections and dead-ends, and edges represent roads.

\subsection{Data Preprocessing}
We first build road networks using structural and feature information from OSM. The edge and node features, and their descriptions, are listed in Table \ref{table:graph-features}. We also extract directional and angular features denoted as 'edge\_attr\_dir' and 'edge\_attr\_ang'. Unlike other graph structures, road networks have these two unique geometric features, which will be discussed in detail in section \ref{section-travel}. 

After collecting the accident records, we run reverse geocoding to find the corresponding addresses of accident coordinates using Nominatim \cite{web23nominatim}. Next, the geocoded data are split according to settlement hierarchy: there are 49 states, 1585 cities, 1644 counties, 3584 towns, and 4227 villages. Since city-level graphs have a medium size, the TAP-city consists of one thousand city-level datasets \footnote{A list of cities sorted by their total counts of traffic accident occurrences can be found in our GitHub repository}. Additionally, we also provided state-level datasets (TAP-state) for 49 U.S. states. The state-level datasets provide a larger graph scale. In general, large and popular cities have more traffic accidents, and their road networks have more nodes and edges. Since the spatial distribution of traffic crashes is sparse and imbalanced, there are limited positive samples, especially for small cities. Therefore, We only select the top 1,000 cities as the positive samples become too sparse. 

Next, missing values are replaced with a new category, and feature data are encoded using one-hot encoding. Then the coordinate data of accident locations are used to find the nearest corresponding nodes in the road networks based on the haversine distance. For the accident occurrence prediction task, binary labels are added to each node indicating whether it contains at least one accident. For the severity prediction task, average accident severities are bucketized into eight classes (using an interval size of 0.5) to be used as labels. The severity feature uses an integer from 0 to 7 to specify the importance of the accident from low to serious impact where 0 denotes no accident and 7 denotes serious impact. Finally, data are split using a stratified split: 60\% of the data is used for training, 20\% is used for validation, and the remaining 20\% is used for testing.

The proposed data repository is well-documented and easily accessible. It can be downloaded with a few lines of code. We are committed to regularly updating the dataset with the latest traffic crash records and auxiliary features. Furthermore, users can effortlessly incorporate additional attributes into the existing data. Our datasets' design enables the straightforward integration of other information, such as weather, human mobility, and points of interest (POIs), thereby facilitating more comprehensive analyses.

\begin{table}
\small
\centering
\caption{Edge features (top) and node features (bottom) included in our datasets.}
\label{table:graph-features}
\begin{tabular}{ll} 
\toprule
\textbf{Graph features} & \textbf{Description}                           \\ 
\midrule
highway                & The type of a road (tertiary, motorway, etc.).  \\
length                 & The length of a road.                           \\
bridge                 & Indicates whether a road represents a bridge.   \\
lanes                  & The number of lanes of a road.                  \\
oneway                 & Indicates whether a road is a one-way street.   \\
maxspeed               & The maximum legal speed limit of a road.        \\
access                 & Describes restrictions on the use of a road.    \\
tunnel                 & Indicates whether a road runs in a tunnel.      \\
junction               & Describes the junction type of a road.          \\ 
edge\_attr\_dir        & Directional information about a road.           \\
edge\_attr\_ang        & Angular information of a road.                  \\
\midrule
highway                & The road type of a node.                        \\
street\_count          & The number of roads connected to a node.        \\
\bottomrule
\end{tabular}
\end{table}

\section{TRAVEL Framework} \label{section-travel}
\subsection{Overview and Motivation} 
Road geometry-related characteristics such as turning radius and direction (i.e., left versus right turns) have long been recognized as important factors affecting road safety \cite{othman2009identifying}. Motivated by this, we design a GNN approach that is effective at capturing information from both road geometry, as well as allowing the use of rich node and edge features already available in OSM.

How do we design a GNN architecture that effectively incorporates road geometry? We find that an effective way to do this is to augment the message passing process in GNNs with additional \emph{angular} and \emph{directional} information. The angular component allows our model to better capture relevant information about an intersection (e.g., whether it has a right or left turn, sharp turns, etc.). The directional component allows the model to capture the direction of a road: for instance, whether it is heading north-to-south versus east-to-west, which can be relevant in practice.

We describe \textsc{TRAVEL} as a layer taking in the previous node embeddings $\mathbf{h}_v$ (suppressing the layer number since we only describe a single \textsc{TRAVEL} layer). 

\begin{figure*}[t]
    \centering
    \includegraphics[width=0.99\textwidth]{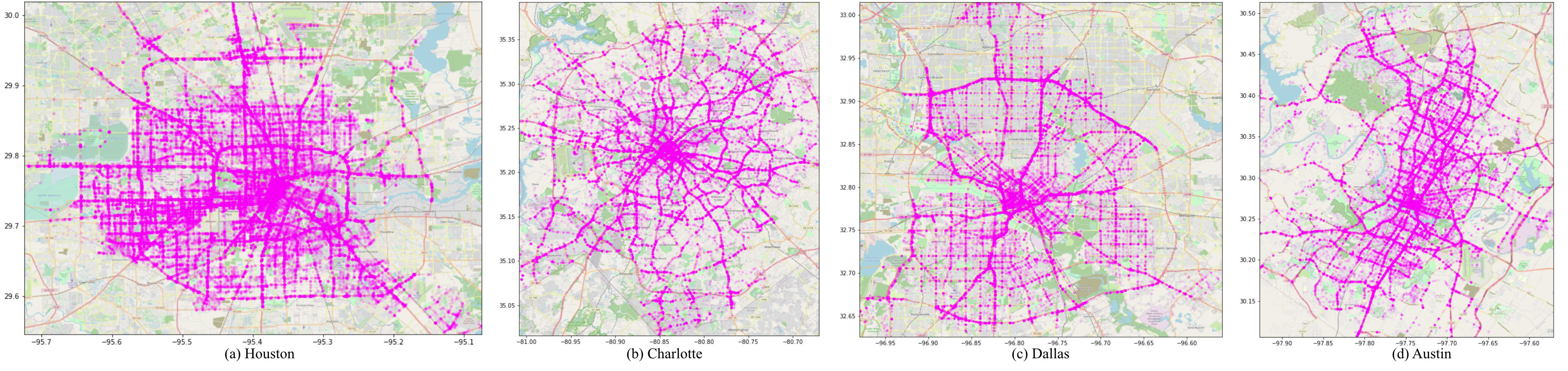}
    \caption{Traffic accident locations of Houston, Charlotte, Dallas, and Austin.}
    \label{fig:four-cities-accidents}
    \Description{Accident locations on the maps of Houston, Charlotte, Dallas, and Austin.}
\end{figure*}

\begin{figure*}[t]
    \centering
    \includegraphics[width=0.99\textwidth]{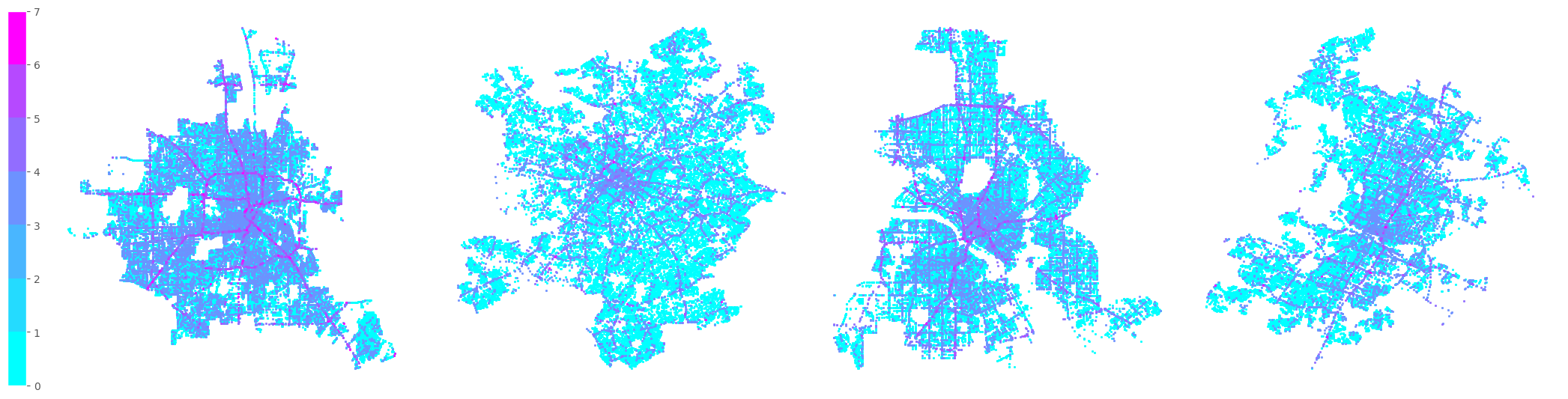}
    \caption{Traffic accident severities of Houston, Charlotte, Dallas, and Austin.}
    \label{fig:four-cities-severities}
    \Description{Visualizations of accident severities on the maps of Houston, Charlotte, Dallas, and Austin.}
\end{figure*}

\subsection{Leveraging Angular information}
The angular component augments the message passing process from node $u$ to node $v$ with information about the angles between the road $(u, v)$ and all the other roads intersecting at node $v$. This allows our GNN model to take into consideration the angles between roads, which are of key importance to road geometry throughout the message passing process.

An important aspect of the design of our angular component is that it is \emph{rotationally symmetric}. This means that the output computed by this component at a node $v$ does not change even if we rotate the graph by any rotation centered at $v$. This makes sense intuitively: the road geometry of a particular intersection does not change if we rotate all roads about this intersection by any fixed angle. This rotational symmetry is important as it ensures that this component has \emph{inductive biases}\footnote{Inductive biases describe the assumptions that a model uses to produce output on unseen data.} that are well suited to its intended task of capturing road geometry.

The angular component is defined as: given points $u$, $v$, $w$, let $\angle(\overrightarrow{uv}, \overrightarrow{wv})$ denote the directed angle\footnote{This is the signed angle that $\overrightarrow{uv}$ must be rotated to have the same direction as $\overrightarrow{wv}$.} from $\overrightarrow{uv}$ to $\overrightarrow{wv}$. Recall that our angular component is designed to augment the messages passed from $u$ to $v$ with information about the angles between road $(u,v)$ and other roads intersecting at $v$. 

Formally, the \emph{set of (directed) angles} between road $(u,v)$ and each of the other roads at $v$ is: 
\begin{align}
    \Phi_{uv} := \{ \angle(\overrightarrow{uv}, \overrightarrow{wv}): w \in \mathcal{N}_v \setminus \{ u \} \}
\end{align}

Next, we will aggregate over $\Phi_{uv}$ to extract suitably summarized information from this set. Rather than using standard aggregation functions, we use a designed aggregation function that aims to particularly emphasize the presence of informative features: namely 1) sharp left turns, 2) sharp right turns, and 3) nearly straight roads. We do this by first defining $\Phi_{uv}^\pi$ as the set $\{|\pi - \phi|: \phi \in \Phi_{uv}\}$, then aggregating as follows (where $\parallel$ denotes concatenation):
\begin{align}
    \mathbf{a}_{uv} := \min(\Phi_{uv}) \parallel \max(\Phi_{uv}) \parallel \min(\Phi^\pi_{uv})
\end{align}

The first two components correspond to the sharpest angles of left and right turns to edge $(u,v)$. The third component corresponds to the angle of $(u,v)$ to the road which is closest to a straight road along with $(u,v)$. Thus, the aggregated angular information $\mathbf{a}_{uv}$ provides a concise summary of the useful information contained in angles between $(u,v)$ and other roads at $v$. 

Finally, our angular component incorporates this angular information $\mathbf{a}_{uv}$ when passing a message along $(u,v)$. The angular component takes in node representations $\mathbf{h}_v$, and outputs the angular node representations $\mathbf{h}_v^{\textrm{Angle}}$:
\begin{align} 
    \mathbf{h}_v^{\textrm{Angle}} &= \mathsf{ReLU}(\mathbf{W} \mathbf{h}_v + \mathbf{m}_{\mathcal{N}(v)}^{\textrm{Angle}}) \label{eq_travel} \\
    \mathbf{m}_{\mathcal{N}(v)}^{\textrm{Angle}} &= \sum_{u \in \mathcal{N}(v)} \mathsf{MLP}(\mathbf{h}_u \parallel \mathbf{e}_{uv} \parallel \mathbf{a}_{uv} ) \label{eq_travel_msg}
\end{align}

An important property of the angular component is that it is rotationally symmetric. We show this as follows.

\begin{theorem}[Rotational Symmetry]
The angular component is rotationally symmetric.
\end{theorem}

\begin{proof}
When rotating all points about $v$ by a fixed rotation, each of the directed angles $\angle(\overrightarrow{uv}, \overrightarrow{wv})$ remains unchanged as $u$ and $w$ rotate by the same angle around $v$. This implies that all the $\mathbf{a}_{uv}$ also remain unchanged. Since the edge features $\mathbf{e}_{uv}$ are also unchanged by the rotation, thus $\mathbf{h}_v^{\textrm{Angle}}$ is also unchanged.
\end{proof}

\subsection{Assessing Directional Attribute}
As we have seen, the angular component is designed to be rotationally symmetric for the purpose of modeling road geometry. However, the directional information ignored by the angular component can still be useful in some contexts: e.g., in some cities, north-south roads may have different characteristics from east-west roads. This motivates our directional component, which captures the direction that each road is heading in. 

Let $\textsc{lat}_u$ and $\textsc{lon}_u$ denote the latitude and longitude of node $u$, respectively. This allows us to compute the \textbf{direction} of the edge $(u,v)$ as:
\begin{equation} \label{eq_ang_theta}
    \mathbf{d}_{uv} = (\textsc{lat}_v - \textsc{lat}_u,\textsc{lon}_v - \textsc{lon}_u)
\end{equation}

Like in the angular component, we incorporate directions into the message passing process:
\begin{align} 
    \mathbf{h}_v^{\textrm{Dir}} &= \mathsf{ReLU}(\mathbf{W} \mathbf{h}_v + \mathbf{m}_{\mathcal{N}(v)}^{\textrm{Dir}}) \label{eq_travel_dir} \\
    \mathbf{m}_{\mathcal{N}(v)}^{\textrm{Dir}} &= \sum_{u \in \mathcal{N}(v)} \mathsf{MLP}(\mathbf{h}_u \parallel \mathbf{e}_{uv} \parallel \mathbf{d}_{uv} ) \label{eq_travel_msg_dir}
\end{align}

Finally, the combined \textsc{TRAVEL} layer's output is the concatenation between the output of the angular and directional components, i.e., $\mathbf{h}_v^{\textrm{Angle}} \parallel \mathbf{h}_v^{\textrm{Dir}}$. This \textsc{TRAVEL} layer can be straightforwardly trained using standard loss functions (cross-entropy loss in our setting) or plugged into any existing GNN.

\section{Experiments}

\begin{table*}
\small
\centering
\caption{Accident occurrence prediction results in terms of F1 score(\%) and AUC(\%).}
\label{table:exp-result}
\begin{tabular}{lcccccccccccc} 
\toprule
            & \multicolumn{2}{c}{Miami}     & \multicolumn{2}{c}{Los Angeles}   & \multicolumn{2}{c}{Orlando}      & \multicolumn{2}{c}{Dallas}      & \multicolumn{2}{c}{Houston}      & \multicolumn{2}{c}{New York}      \\
\cmidrule(lr){2-3}\cmidrule(lr){4-5}\cmidrule(lr){6-7}\cmidrule(lr){8-9}\cmidrule(lr){10-11}\cmidrule(lr){12-13}
\textbf{Classifier}                         & F1             & AUC            & F1             & AUC            & F1             & AUC            & F1             & AUC            & F1             & AUC            & F1             & AUC                         \\ 
\midrule

XGBoost & 11.8±1.6 & 53.1±0.4 & 16.5±0.4 & 54.5±0.1 & 39.4±1.1 & 61.4±0.3 & 31.0±2.2 & 58.5±0.8 & 16.1±0.6 & 53.8±0.2 & 23.8±0.8 & 56.8±0.3 \\
MLP         & 13.0±0.8          & 61.3±1.9          & 16.0±0.5          & 66.3±0.1          & 38.8±1.9          & 65.6±2.0          & 32.5±1.1          & 67.8±0.4          & 15.9±0.7          & 64.0±0.4          & 23.7±0.9          & 65.6±1.2          \\
GCN         & 20.0±3.3          & 68.5±3.3          & 40.2±1.1          & 80.4±0.3          & 51.6±0.8          & 73.1±1.2          & 39.8±1.9          & 73.1±0.4          & 16.4±1.3          & 66.7±0.2          & 39.2±3.7          & 75.5±0.4          \\
ChebNet     & 20.7±2.9          & 71.3±3.6          & 39.8±1.8          & 81.0±0.3          & 53.1±0.6          & 76.7±1.6          & 42.0±0.5          & 75.8±0.4          & 23.8±0.5          & 69.6±0.5          & 40.9±4.3          & 78.3±1.1          \\
ARMANet     & 19.2±3.3          & 69.5±3.5          & 40.8±1.0          & 80.9±0.4          & 51.5±1.3          & 75.7±1.4          & 41.2±0.5          & 75.6±0.2          & 23.1±0.4          & 69.2±0.7          & 42.4±1.1          & 77.7±0.6          \\
GraphSAGE   & 20.7±2.4          & 67.6±2.8          & 41.6±0.5          & 80.5±0.3          & 52.6±1.3          & 74.1±1.2          & 44.2±0.5          & 74.4±0.3          & 23.7±0.4          & 68.5±0.4          & 42.5±1.1          & 76.3±0.1          \\
TAGCN       & 25.2±1.1          & 73.5±2.4          & 49.5±0.7          & 84.7±0.2          & 53.3±2.5          & 77.2±1.2          & 45.4±0.4          & 77.0±0.5          & 23.7±0.6          & 70.5±0.3          & 42.0±1.1          & 81.5±0.2          \\
GIN         & 22.8±1.2          & 72.7±2.6          & 41.6±0.7          & 81.8±0.2          & 54.7±1.4          & 76.6±1.1          & 41.3±2.0          & 75.2±0.3          & 20.9±1.0          & 68.0±0.3          & 41.7±2.1          & 79.1±0.5          \\
GAT         & 22.6±1.5          & 68.3±3.0          & 41.6±0.4          & 80.9±0.2          & 55.3±1.3          & 74.1±1.0          & 42.1±1.5          & 73.6±0.3          & 17.8±0.8          & 67.3±0.3          & 42.2±0.5          & 76.6±0.4          \\
MPNN        & 38.8±2.1          & 82.4±1.0          & 46.0±1.6          & 83.9±0.2          & 61.4±2.5          & 81.8±0.7          & 48.5±1.9          & 79.4±0.4          & 28.2±1.7          & 73.5±0.5          & 44.9±0.8          & 86.9±0.4          \\
CGC         & 34.4±2.7          & 79.5±1.5          & 45.0±1.2          & 81.5±0.2          & 59.0±2.1          & 81.1±0.8          & 48.5±0.5          & 79.2±0.7          & 27.3±1.9          & 72.3±0.1          & 40.6±1.2          & 85.4±0.8          \\
Transformer & 37.7±3.3          & 81.0±1.9          & 48.9±0.3          & 83.8±0.3          & 62.9±1.6          & 82.0±0.7          & 49.8±0.7          & 80.0±0.7          & 28.4±0.7          & 73.9±0.4          & 43.1±0.7          & 87.2±0.4          \\
GEN         & 44.9±3.1          & 81.0±2.4          & 48.6±6.2          & 82.7±0.9          & 63.0±1.1          & 81.2±0.9          & 56.5±1.7          & 79.5±0.1          & 34.1±6.0          & 73.7±0.4          & 47.3±1.4          & 87.7±0.9          \\
{TRAVEL}      & {51.9±1.0} & {84.9±0.9} & {55.3±0.9} & {85.9±0.5} & {65.0±0.4} & {82.3±0.4} & {58.0±0.9} & {80.8±0.7} & {46.4±0.7} & {74.5±0.3} & {51.1±0.9} & {88.2±0.2} \\

\bottomrule
\end{tabular}
\end{table*}

\begin{figure*}
    \centering
    \includegraphics[width=0.95\textwidth]{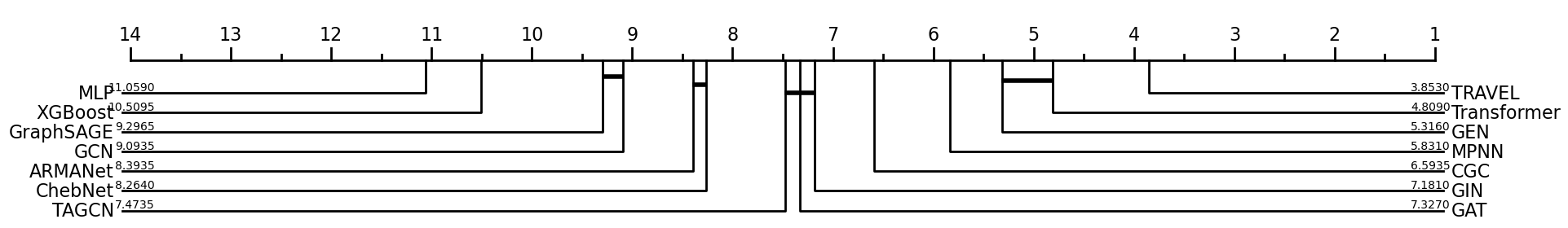}
    \caption{Critical difference diagram of the traffic accident occurrence prediction task. TRAVEL performs statistically significantly better than all other methods.}
    \label{fig:cd-diagram-occurrence}
    \Description{A critical difference diagram showing the pairwise statistical difference comparison of all models on one thousand different datasets for the accident occurrence prediction task.}
\end{figure*}

In this section, we comprehensively evaluate the performance of \textsc{TRAVEL} and thirteen other baselines, finding that \textsc{TRAVEL} clearly outperforms them on the traffic accident occurrence and severity prediction tasks. Figure \ref{fig:four-cities-accidents} and \ref{fig:four-cities-severities} plot the labels to be predicted for the two tasks respectively.

From Table \ref{table:data-stat-city} and Table \ref{table:data-stat-state}, we observe that the traffic accident data are imbalanced: only a small proportion of the nodes have accidents. Therefore, the models are evaluated based on F1 score, Area Under the Receiver Operating Characteristic Curve (AUC), and Accuracy.

All neural network models are implemented using PyTorch \cite{paszke2019pytorch} version 1.12.1 and PyTorch Geometric \cite{fey2019pytorchgeo} version 2.2.0. The models are trained using the cross-entropy loss and Adam optimizer. We train models for 300 epochs and use 16 as their hidden dimensions. The experiments are run on a machine with AMD(R) EPYC 7352 and NVIDIA(R) RTX A5000 graphics card (CUDA 11.6). 

All GNN models have two graph convolutional layers, where layer-2 embeddings get information from nodes two hops away. We also add a fully connected layer after GNN layers to increase their expressive power. Moreover, L2 penalty and dropout layers with 0.5 dropout rates are added to all neural network models to reduce overfitting.

\subsection{Baselines}
We compare our \textsc{TRAVEL} model with two generic machine learning models and eleven state-of-the-art GNN models: (1) \textbf{XGBoost}: A scalable gradient boosting tree-based approach \cite{chen2016xgboost}. (2) \textbf{MLP}: A multi-layer perceptron is a classic feedforward artificial neural network. (3) \textbf{GCN}: Graph Convolutional Networks, which generalize CNNs from low-dimensional regular grids to graph data using neighborhood-based filters \cite{kipf2016gcnconv}. (4) \textbf{ChebNet}: Chebyshev spectral graph convolution networks, which are spectral graph convolutional architectures with fast localized spectral filtering \cite{defferrard2016chebnet}. (5) \textbf{ARMANet}: Graph neural networks with convolutional auto-regressive moving average (ARMA) filters \cite{bianchi2021armanet}. (6) \textbf{GraphSAGE}: A general inductive framework for inductive representation learning on graphs \cite{hamilton2017inductive}. (7) \textbf{TAGCN}: Topology adaptive graph convolutional networks, which use fixed-size learnable filters to perform convolutions on graphs \cite{du2018tagcn}. (8) \textbf{GIN}: Graph Isomorphism Networks, which generalize the Weisfeiler-Lehman (WL) graph isomorphism test \cite{xu2018gnn,weisfeiler1968reduction}. (9) \textbf{GAT}: Graph attention networks, which apply attentional mechanisms during aggregation \cite{velivckovic2017gat}. (10) \textbf{MPNN}: Message Passing Neural Network, which is a general GNN framework designed for computational chemistry, reasoning, and simulation. \cite{gilmer2017mpnn}. (11) \textbf{CGC}: Crystal graph convolutional neural network is an accurate and interpretable framework that can extract the contributions from local features to global properties \cite{xie2018cgc}. (12) \textbf{GEN}: GENeralized graph convolutional neural networks, which support softmax, power, and mean aggregation \cite{li2020gen}. (13) \textbf{Transformer}: Graph transformers, which adopt vanilla multi-head attention into graph learning with taking into account the case of edge features \cite{shi2021graphtransformer}.

\begin{table*}
\small
\centering
\caption{State-level accident occurrence prediction results in terms of F1 score(\%), and AUC(\%).}
\label{table:exp-result-state}
\begin{tabular}{lcccccccccccc} 
\toprule
            & \multicolumn{2}{c}{California}     & \multicolumn{2}{c}{Oregon}   & \multicolumn{2}{c}{Utah}      & \multicolumn{2}{c}{Maryland}      & \multicolumn{2}{c}{Minnesota}      & \multicolumn{2}{c}{Connecticut}      \\
\cmidrule(lr){2-3}\cmidrule(lr){4-5}\cmidrule(lr){6-7}\cmidrule(lr){8-9}\cmidrule(lr){10-11}\cmidrule(lr){12-13}
\textbf{Classifier}                         & F1             & AUC            & F1             & AUC            & F1             & AUC            & F1             & AUC            & F1             & AUC            & F1             & AUC                         \\ 
\midrule
XGBoost     & 10.4±0.3      & 52.7±0.1       & 10.5±0.6  & 52.8±0.2   & 25.3±0.5 & 57.7±0.2 & 17.8±0.4    & 54.9±0.1     & 14.4±1.1     & 53.9±0.3      & 26.2±1.7       & 57.7±0.6        \\
MLP         & 10.2±0.2      & 64.8±0.1       & 9.0±0.3   & 60.8±0.3   & 26.7±0.2 & 63.9±0.5 & 17.8±0.2    & 69.5±0.1     & 14.4±1.1     & 63.8±0.4      & 26.1±1.9       & 65.3±0.9        \\
GCN         & 24.0±0.0      & 71.5±0.0       & 20.6±0.5  & 68.7±0.7   & 32.7±0.1 & 76.3±0.3 & 26.1±1.3    & 79.5±0.4     & 28.1±0.4     & 70.9±0.3      & 40.0±0.6       & 76.2±0.8        \\
ChebNet     & 23.2±1.0      & 72.9±0.2       & 21.0±0.2  & 73.1±0.3   & 34.3±1.2 & 77.3±0.5 & 28.5±0.3    & 80.4±0.1     & 30.2±2.1     & 74.1±1.3      & 42.0±0.4       & 76.6±0.2        \\
ARMANet     & 23.6±2.0      & 72.8±0.2       & 18.6±3.4  & 72.7±0.7   & 34.6±0.3 & 77.2±0.3 & 28.6±1.6    & 80.6±0.2     & 26.4±1.7     & 72.7±1.2      & 42.4±1.5       & 77.2±0.6        \\
GraphSAGE   & 25.8±0.4      & 72.8±0.4       & 21.4±0.9  & 71.2±1.3   & 34.3±1.5 & 77.7±0.5 & 28.5±1.2    & 80.2±0.1     & 28.9±0.1     & 71.9±0.8      & 42.7±1.6       & 77.2±0.6        \\
TAGCN       & 28.7±0.4      & 75.9±0.1       & 24.7±1.0  & 76.6±0.1   & 34.1±1.0 & 78.9±0.4 & 26.1±0.9    & 81.8±0.4     & 30.8±0.9     & 77.2±0.8      & 37.8±0.8       & 78.2±1.3        \\
GIN         & 28.0±0.2      & 72.7±0.2       & 24.3±0.4  & 74.2±0.2   & 36.2±0.3 & 78.9±0.5 & 28.2±0.6    & 80.8±0.2     & 32.0±1.8     & 74.9±1.3      & 42.0±1.2       & 77.2±0.4        \\
GAT         & 24.5±0.3      & 72.2±0.1       & 21.8±0.4  & 70.2±1.0   & 35.6±2.2 & 76.8±0.6 & 27.8±0.2    & 80.4±0.2     & 29.1±0.2     & 71.3±1.1      & 41.9±2.6       & 75.9±1.5        \\
MPNN        & 33.2±1.7      & 79.9±0.5       & 56.2±1.3  & 89.2±0.8   & 43.1±2.6 & 88.7±0.3 & 32.3±0.2    & 89.4±0.2     & 38.5±1.0     & 89.8±0.1      & 43.9±0.6       & 85.8±0.7        \\
CGC         & 34.4±2.8      & 79.0±0.1       & 53.1±3.3  & 88.5±0.1   & 47.1±0.8 & 88.3±0.2 & 43.0±1.9    & 88.6±0.4     & 48.1±1.7     & 88.7±0.5      & 43.1±0.2       & 85.9±0.7        \\
Transformer & 29.3±0.5      & 80.1±0.1       & 53.5±0.4  & 89.4±0.3   & 44.8±1.4 & 89.0±0.1 & 31.5±0.6    & 89.2±0.6     & 33.9±0.5     & 90.1±0.2      & 43.9±0.7       & 85.1±0.6        \\
GEN         & 43.9±0.1      & 77.3±0.4       & 59.4±0.8  & 87.2±0.7   & 53.8±0.8 & 86.1±0.3 & 40.8±2.8    & 88.4±0.8     & 52.3±1.5     & 88.7±0.4      & 43.1±1.1       & 81.8±0.7        \\
TRAVEL      & 46.1±0.7      & 81.1±1.0       & 60.8±0.0  & 90.5±0.2   & 55.6±0.6 & 89.9±1.9 & 46.4±1.4    & 90.1±0.3     & 55.2±1.0     & 91.3±0.7      & 46.6±0.3       & 85.2±0.6        \\

\bottomrule
\end{tabular}
\end{table*}

Among GNN baselines, GCN, ChebNet, ARMANet, GraphSAGE, TAGCN, and GIN do not support message passing with multi-dimensional edge features. In contrast, GAT, MPNN, CGC, GEN, and Graph Transformer support message passing with multi-dimensional edge features.

\begin{table*}
\centering
\small
\caption{Accident severity prediction results in terms of weighted F1 score(\%).}  
\label{table:exp-result-severity}
\begin{tabular}{lcccccccccc} 
\toprule
\textbf{Classifier}  & Miami  & Los Angeles  & Orlando  & Dallas  & Houston  & New York  & Charlotte  & San Diego  & Nashville  & Sacramento  \\ 
\midrule

XGBoost     & 81.3±0.1          & 82.4±0.0          & 66.9±1.4          & 70.6±0.1          & 71.0±0.1          & 88.8±0.1          & 78.4±0.4          & 79.9±0.2          & 75.7±0.3          & 83.9±0.5          \\
MLP         & 81.3±0.1          & 80.9±0.0          & 66.7±1.2          & 70.3±0.2          & 68.3±0.0          & 88.1±0.6          & 78.1±0.2          & 79.4±0.2          & 75.6±0.4          & 83.7±0.5          \\
GCN         & 81.1±0.3          & 84.7±0.2          & 68.6±0.3          & 67.3±0.2          & 69.3±0.1          & 88.9±1.1          & 80.4±0.4          & 82.5±0.3          & 78.8±0.2          & 85.5±0.4          \\
ChebNet     & 81.4±0.2          & 84.9±0.1          & 68.5±0.9          & 72.4±0.4          & 69.4±0.9          & 89.8±0.2          & 81.0±0.6          & 82.1±0.1          & 79.0±0.5          & 85.7±0.6          \\
ARMANet     & 81.4±0.5          & 84.9±0.1          & 68.4±1.0          & 72.2±0.3          & 69.4±0.1          & 89.8±0.1          & 81.1±0.6          & 82.2±0.3          & 79.1±0.4          & 85.7±0.7          \\
GraphSAGE   & 81.4±0.2          & 84.7±0.2          & 67.9±0.8          & 72.1±0.2          & 70.2±0.4          & 89.8±0.1          & 81.0±0.7          & 82.6±0.2          & 79.2±0.5          & 86.1±0.6          \\
TAGCN       & 81.3±0.4          & 86.2±0.2          & 68.2±1.1          & 72.6±0.1          & 69.1±0.2          & 88.9±0.2          & 81.1±0.7          & 83.9±0.4          & 79.2±0.3          & 86.4±0.4          \\
GIN         & 81.3±0.4          & 85.4±0.2          & 67.5±1.1          & 71.6±0.3          & 69.1±0.2          & 89.4±0.3          & 80.8±0.6          & 82.9±0.2          & 79.2±0.3          & 86.2±0.4          \\
GAT         & 81.5±0.6          & 85.3±0.2          & 70.3±0.5          & 71.6±0.4          & 68.4±0.3          & 88.5±0.8          & 81.8±0.6          & 82.2±0.1          & 79.5±0.2          & 85.9±0.6          \\
MPNN        & 82.2±0.5          & 85.3±0.4          & 72.1±1.2          & 73.8±0.3          & 70.8±0.4          & 89.9±0.1          & 84.5±0.4          & 84.3±0.6          & 82.7±0.3          & 86.3±0.6          \\
CGC         & 82.6±0.6          & 85.8±0.1          & 72.0±0.8          & 74.2±0.1          & 71.2±0.3          & 89.8±0.3          & 83.6±0.4          & 84.3±0.5          & 82.9±0.4          & 85.8±0.5          \\
Transformer & 83.4±0.2          & 85.8±0.1          & 73.0±1.2          & 73.9±0.2          & 71.3±0.3          & 89.9±0.2          & 84.8±0.0          & 84.3±0.4          & 82.9±0.3          & 86.5±0.8          \\
GEN         & 83.6±0.6          & 85.3±0.3          & 72.8±1.3          & 74.8±0.1          & 69.7±0.7          & 89.4±0.1          & 84.5±0.3          & 83.7±0.3          & 82.7±0.3          & 86.8±0.3          \\
TRAVEL      & {84.7±1.1} & {87.2±0.3} & {73.8±0.9} & {75.7±0.2} & {74.5±0.6} & {90.5±0.2} & {85.2±0.1} & {85.4±0.5} & {83.5±0.3} & {87.2±0.2} \\     

\bottomrule
\end{tabular}
\end{table*}

\subsection{Traffic Accident Occurrence Prediction}
In the accident occurrence prediction task, which is formulated as a node classification problem, we aim to predict whether a node has traffic accidents or not based on previous accident records. Table \ref{table:exp-result} shows the prediction results on the six major U.S. cities. We run every experiment three times and report the average score along with the standard deviation in the format of "average score ± standard deviation". We generally observe that:
(1) The proposed \textsc{TRAVEL} consistently achieves the best performance on all the metrics, due to its ability to capture angular and directional features on top of other environmental features.
(2) GNN-based approaches generally outperform XGBoost and MLP. This is because nodes in GNNs can aggregate feature information from their neighbors, while the MLP and XGBoost can only learn from local feature data. 
(3) GNN variants that support multi-dimensional edge features generally outperform models that do not support them. 

To evaluate performance across the full set of one thousand city-level datasets, in Figure \ref{fig:cd-diagram-occurrence}, we use the Wilcoxon-Holm critical difference diagram\cite{IsmailFawaz2018cddiagram}. It can be interpreted as follows: methods are arranged by their average rank, so the rightmost method (\textsc{TRAVEL}) is the one with the overall best performance in terms of F1 score, i.e., lowest average rank. Meanwhile, the thick horizontal lines group a set of classifiers that are not significantly different in performance from one another. Thus, since \textsc{TRAVEL} is not connected by any horizontal lines, this indicates that it performs significantly better than all the other methods.

\subsection{Traffic Accident Severity Prediction}
The goal of this task is to predict the future accident severities of accidents. Since a node may have multiple accidents with different severity, we compute each node's mean severity and bucketize this mean severity into eight classes (severity is represented by a number between 0 and 7, where 0 denotes no accident, 1 indicates the most negligible impact on traffic, and 7 indicates a significant impact on traffic.). Figure \ref{fig:four-cities-severities} suggests that accidents near hectic main roads tend to have higher severity. As shown in Table \ref{table:exp-result-severity}, our \textsc{TRAVEL} model again clearly outperforms the baselines across all datasets in terms of weighted F1 score.

\section{Potential Applications}

\subsection{Challenges}
Data availability, integration, sparsity, and imbalanced distribution pose challenges to using traffic data effectively. Firstly, traffic data are often controlled by governments and private companies for privacy and security reasons. In fact, as of May 2018, only 17 cities in the US have open web portals of transportation data for public access \cite{zhu2019availability}. Moreover, as of April 2023, 24 states in the US are equipped with red light cameras and only 20 states have speed cameras in place. Traffic monitoring devices are completely prohibited by state laws in 9 states \cite{web2023lawbystate}. Secondly, the available open data often use different formats and features, making data integration difficult. Thirdly, data sparsity and imbalanced distribution can lead to bias, as traffic crashes are unevenly distributed across locations and time periods. The proposed TAP repository, with its extensive coverage, real-world road structure data, and multi-dimensional geospatial features, has the potential to benefit the research community in a variety of applications, such as traffic crash prediction, road safety analysis, and traffic crash mitigation.

\subsection{Traffic Crash Prediction}
The lack of comprehensive and reliable datasets is a major challenge in developing effective models for roadway crash prediction. The existing datasets are often incomplete, limited in scope, and outdated. In contrast, the TAP data repository contains up-to-date crash data with nationwide coverage. It also contains geospatial features on top of real-world graph structure data. Covering over 1,000 major U.S. cities and 49 states, the proposed city-level and state-level datasets are ideal for a variety of city-level and state-level traffic crash prediction tasks. Moreover, it contains the angular and directional features that we extract from road networks. By leveraging such data, it is possible to develop accurate models for roadway crash prediction, which in turn can facilitate the development and implementation of effective measures to mitigate traffic accidents.

We have discussed accident occurrence prediction and severity prediction. Other potential tasks include crash risk prediction and frequency prediction \cite{chen2016realtimeaccident,yuan2018convlstmaccident}. In addition, included accident timestamps information makes it suitable for spatial-temporal crash prediction tasks. Furthermore, the underlying graph structures not only facilitate the application of graph-based machine learning methods but also enable the modeling and simulation of transportation systems following real-world road topology. This work mainly discusses static geospatial features. One of the open problems is incorporating dynamic traffic flow data to crash prediction. Traffic flow data also suffers from availability and sparsity issues.

\subsection{Road Safety Analysis}
In addition to its potential for traffic accident prediction, our datasets are valuable for road safety analysis, particularly with respect to environmental causes. Traffic systems are complex, involving a variety of entities such as vehicles, motorists, pedestrians, bicyclists, road infrastructure, and the environment. Research has shown that police and the general public possess a common understanding of the top causes of traffic crashes, including human-related risk factors such as driver distraction, drug and alcohol impairment, drowsy driving, and speeding \cite{tefft2016crashsleep,rolison2018factors}. By contrast, environmental factors such as turning radius and direction may be less obvious, and rigorous analyses are needed to identify underlying patterns. Such crash analyses can be useful in developing education campaigns and directing funds towards reducing traffic fatalities.

Designing roadway environments can play a key role in mitigating human errors and accounting for injury tolerances \cite{web23countermeasure}. Indeed, the environment surrounding the roadway system, including factors such as land use and the intersections of highways, roads, and streets with other transportation modes (e.g., rail and transit), strongly influences the safety risks faced by the traveling public. Thus, it is essential to carefully consider roadway design in order to ensure that roadways are used safely and effectively.

\subsection{Traffic Crash Mitigation}
As one of the most pressing global health issues, governments, organizations, and individuals have been trying to alleviate traffic crash risk for many years. The insights derived from traffic crash data are therefore of significant importance for policymakers in making evidence-based decisions. Crash prediction and analysis can improve road safety in numerous ways. Governments will gain a clear understanding of accident hotspots, while the general public who uses map services will know when and where to avoid potential accidents. Furthermore, travel time prediction and route planning are highly linked to crash prediction since the occurrence of a traffic crash can disrupt nearby traffic flow.

Several measures taken by governments have been demonstrated to be effective \cite{web23countermeasure}, including improvements at curves, corridor access management, reduced left-turn conflict intersections, and intersections with a circular configuration (roundabout). Roadway design has a substantial influence on the traveling public. However, given these countermeasures, another challenge is how to make the most of existing infrastructure since upgrading infrastructure or altering existing road topology can be cost-prohibitive. Thus, a more practical approach is needed to design a redundant system to avoid human errors that could lead to fatalities.

\section{Conclusion}
In this paper, we affirm the benefits of GNNs for the critically important task of roadway crash prediction. We first formulate the accident occurrence prediction and accident severity prediction tasks as graph-based node classification problems. To stimulate future research in this field, we construct and release a comprehensive graph-based Traffic Accident Prediction (TAP) data repository with these two benchmark tasks. We also evaluate thirteen state-of-the-art machine learning approaches using the proposed data. The proposed datasets feature real-world graph structures and geospatial data, making them a reliable, comprehensive, and user-friendly resource for traffic crash analysis. Furthermore, we propose our \textsc{TRAVEL} framework, designed to capture angular and directional attributes, and prove its theoretical property (rotational symmetry of angular component). The experiments show that \textsc{TRAVEL} consistently outperforms the baselines. Practically, \textsc{TRAVEL} requires only features available in OpenStreetMaps as input, so it can be readily applied to almost all cities. For future work, we plan to investigate model explainability.



\bibliographystyle{ACM-Reference-Format}
\bibliography{biblio}


\end{document}